\documentclass{article}
\usepackage{times}
\usepackage{xcolor}
\usepackage{soul}
\usepackage[utf8]{inputenc}
\usepackage[small]{caption}

\usepackage{amsmath,amsfonts,amssymb,theorem,enumerate,latexsym}
\usepackage[capitalise, nameinlink]{cleveref}

\usepackage{subcaption}
\usepackage{graphicx}
\usepackage{tikz}
\usepackage{wrapfig}
\usepackage{algorithm}
\usepackage{algorithmic}
\newcommand{\commentout}[1]{}

\newcommand{\new}{\textit{new}}
\newcommand{\old}{\textit{old}}

\newcommand{\Stay}{\mathit{Stay}}
\newcommand{\Go}{\mathit{Go}}
\newcommand{\Up}{\mathit{Up}}
\newcommand{\Down}{\mathit{Down}}

\newcommand{\calM}{\mathcal{M}}

\newcommand{\reals}{\mathbb{R}}
\newcommand{\E}{\mathop{\mathbb E}}

\newcommand{\veps}{\varepsilon}

\newcommand{\denselist}{\itemsep 0pt\partopsep 0pt}

\newcommand{\scite}{\cite}

\newtheorem{cor}{Corollary}

\newtheorem{obs}{Observation}

\newtheorem{lemma}{Lemma}
\newtheorem{theorem}{Theorem}
\newtheorem{thm}[theorem]{Theorem}

\newtheorem{corollary}[theorem]{Corollary}
\def\QuadSpace{\vspace{0.25\baselineskip}}
\def\HalfSpace{\vspace{0.5\baselineskip}}

\def\EndProof{ \hfill \vrule width 1ex height 1ex depth 0pt \newline }
\newenvironment{proof}{\QuadSpace\par\noindent{\bf Proof}:}{\EndProof\HalfSpace \vspace{-0.15in}}
\newenvironment{sketch}{\QuadSpace\par\noindent{\bf Proof Sketch}:}{\EndProof\HalfSpace \vspace{-0.15in}}

\DeclareMathOperator*{\argmax}{arg\,max}

\newcommand{\ind}{\mathbf{1}}

\title{Planning and Learning with Stochastic Action Sets}

\author{
  Craig Boutilier,
  Alon Cohen,
  Amit Daniely,
  Avinatan Hassidim,\\
  \textbf{Yishay Mansour,
  Ofer Meshi,
  Martin Mladenov,
  Dale Schuurmans}
  \\
  Google Research, Mountain View, CA, USA\\
{\small{\texttt{\{cboutilier,aloncohen,amitdaniely,avinatan,mansour,meshi,schuurmans\}@google.com}}}
}

\begin{document}

\maketitle

\begin{abstract}
In many practical uses of reinforcement learning (RL)
the set of actions available at a given state is a random variable,
with realizations governed by an exogenous stochastic process.
Somewhat surprisingly, the foundations for such sequential decision processes
have been unaddressed.
In this work, we formalize and investigate 
\emph{MDPs with stochastic action sets (SAS-MDPs)} 
to provide these foundations.
We show that optimal policies and value functions in this model have a
structure that admits a compact representation.
From an RL perspective, we show that Q-learning with sampled action sets 
is sound.
In model-based settings, we consider two important special cases:
when individual actions are available with independent probabilities; and a
sampling-based model for unknown distributions. We develop 
poly-time value and policy iteration methods for both cases; and in the first,
we offer a poly-time linear programming solution.
\end{abstract}

\section{Introduction}
\label{sec:intro}

Markov decision processes (MDPs) are the standard model for sequential 
decision making under uncertainty,
and provide the foundations for reinforcement learning (RL).
With the recent emergence of RL as a practical AI technology
in combination with deep learning \cite{mnih2013,mnih2015},
new use cases are arising that challenge basic MDP modeling assumptions.
%
One such challenge is that many practical MDP and RL problems have 
\emph{stochastic sets of feasible actions};
that is, the set $A_s$ of feasible actions at state $s$ 
\emph{varies stochastically} with each visit to $s$.
For instance, in online advertising, the set of available ads
differs at distinct occurrences of the same state
(e.g., same query, user, contextual features),
due to exogenous factors like 
campaign expiration or budget throttling.
In recommender systems with large item spaces, 
often a set of \emph{candidate} recommendations is first generated,
from which top scoring items are chosen;
exogenous factors often induce non-trivial changes in the candidate set.
With the recent application of MDP and RL models in
ad serving and recommendation 
\cite{charikar:stoc99,Li:adkdd2009,archak-mirrokni-muthu:www10,mirrokni:wine12,kearns:uai12,silver:icml13,theocharous:ijcai15,logisticMDPs:ijcai17}, 
understanding how
to capture the stochastic nature of available action sets is critical.

Somewhat surprisingly, this problem seems to have been largely
unaddressed in the literature.
Standard MDP formulations \cite{puterman}
allow each state $s$ to have its own feasible action set $A_s$,
and it is not uncommon to allow the set $A_s$
to be non-stationary or time-dependent.
However, they do not support
the treatment of $A_s$ as a stochastic random variable.
In this work, we: (a) introduce the \emph{stochastic action set MDP (SAS-MDP)}
and provide its theoretical foundations;
(b) describe how to account for stochastic action sets
in model-free RL (e.g.,
Q-learning); and (c) develop tractable algorithms for solving SAS-MDPs in
important special cases.

An obvious way to treat this problem is to embed the 
set of available actions into the state itself.
This provides a useful analytical tool, 
but it does not immediately provide tractable algorithms for learning
and optimization, since each state is augmented with all possible 
\emph{subsets} of actions, 
incurring an exponential blow up in state space size.
To address this issue,
we show that SAS-MDPs possess an important
property: the Q-value of an available action $a$ is
independent of the availability of other actions.
This allows us to prove that optimal policies can be represented compactly
using (state-specific) decision lists (or orderings) over the action set.

This special structure allows one to solve the SAS RL problem effectively
using, for example, Q-learning.  We also devise model-based algorithms that
exploit this policy structure. We develop value and policy iteration
schemes, showing they converge in a polynomial number of iterations
(w.r.t.\ the size of the underlying ``base'' MDP). We also
show that per-iteration complexity is polynomial time for two important 
special forms of action availability distribution: (a)
when action availabilities are independent, both methods are exact;
(b) when the distribution over sets $A_s$ is sampleable, we obtain
approximation algorithms with polynomial sample complexity. In fact,
policy iteration is strongly polynomial under additional 
assumptions (for a fixed discount factor).
We show that a linear program for SAS-MDPs can be
solved in polynomial time as well.
Finally, we offer a simple empirical demonstration of the importance
of accounting for stochastic action availability when computing
an MDP policy.

Additional discussion and full proofs of all results can be found in
a longer version of this paper \cite{sasmdps_full:arxiv18}.

\section{MDPs with Stochastic Action Sets}
\label{sec:sas}

We first introduce SAS-MDPs and provide a simple example illustrating how
action availability impacts optimal decisions.
See \cite{puterman} for more
background on MDPs.

\subsection{The SAS-MDP Model}
\label{sec:sasmodel}

Our formulation of \emph{MDPs with Stochastic Action Sets (SAS-MDPs)} derives 
from a standard, finite-state, finite-action MDP (the \emph{base MDP}) $\calM$,
with $n$ states $S$, \emph{base} actions $B_s$ for $s\in S$,
and transition and reward functions,
$P: S \times B \rightarrow \Delta(S)$ and $r:S\times B \rightarrow \reals$.
We use $p^k_{s,s'}$ and $r^k_s$ to denote the probability of transition to
$s'$ and the accrued reward, 
respectively, when action $k$ is taken at state $s$.
For notational ease, we assume that feasible action sets for each $s\in S$ 
are identical, so $B_s = B$ (allowing distinct base sets at different states has no impact on what follows).
Let $|B| = m$ and $M = |S \times B | =nm$.
We assume an infinite-horizon, discounted objective with fixed
discount rate $\gamma$, $0\leq\gamma<1$.

In a SAS-MDP, the set of actions available at state $s$ at any stage $t$
is a random subset $A^{(t)}_s \subseteq B$.
We assume a family of \emph{action availability distributions}
$P_s\in\Delta(2^B)$
defined over the powerset of $B$.
These can depend on $s\in S$ but are otherwise history-independent, hence
$\Pr(A^{(t)}_s | s^{(1)},\ldots,s^{(t)}) = \Pr(A^{(t)}_s | s^{(t)})$.
Only actions $k\in A^{(t)}_s$ in the realized available action set 
can be executed  at stage $t$.
Apart from this, the dynamics of the MDP is unchanged:
when an (available) action is taken, state transitions and
rewards are prescribed as in the base MDP.
In what follows, we assume that some action is always available,
i.e., $\Pr(A^{(t)}_s = \emptyset) = 0$ for all $s, t$.\footnote{Models
that trigger process termination when $A^{(t)}_s = \emptyset$ are well-defined,
but we set aside this model variant here.} Note that 
a SAS-MDP does not conform to the usual definition
of an MDP.

\subsection{Example}
\label{sec:sasexample}

The following simple MDP shows the importance of accounting for stochastic
action availability when making decisions.
The MDP below has two states. Assume the agent starts at state $s_1$, where
two actions (indicated by directed edges for their transitions) are always 
available: one ($\Stay$) stays at $s_1$, and the
other ($\Go$)
transitions to state $s_2$, both with reward $1/2$. At $s_2$, the action
$\Down$
returns to $s_1$, is always available and has reward 0. A second action
$\Up$ also returns to $s_1$, but is available with
only probability $p$ and has reward 1.
\includegraphics[width=0.5\columnwidth]{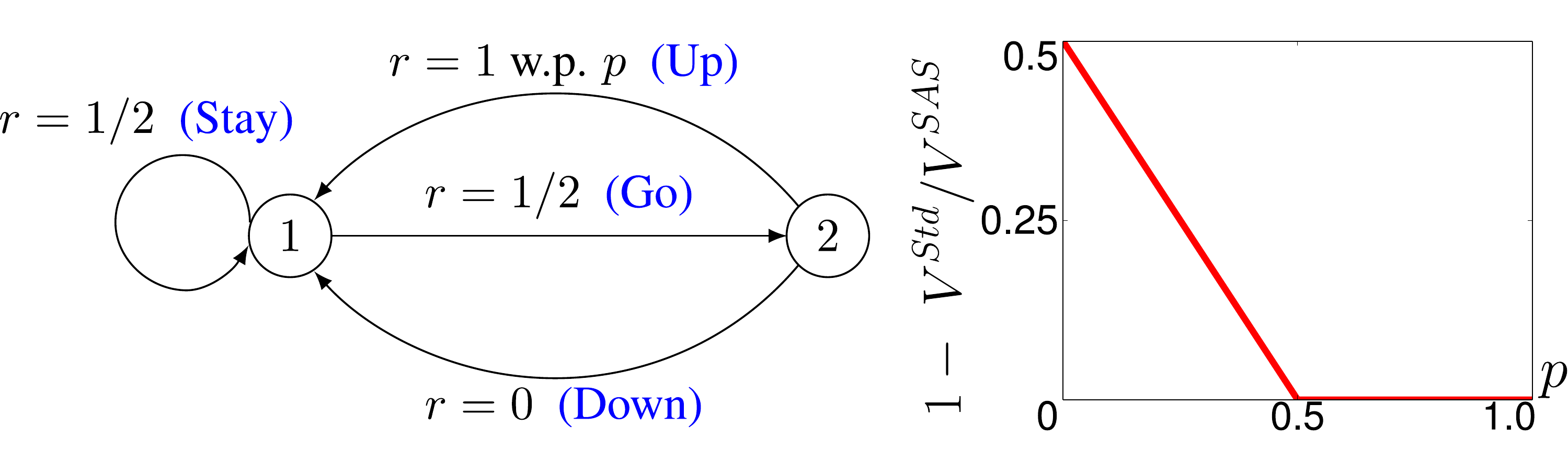}



A naive solution that ignores action availability is as follows:
we first compute the optimal $Q$-function assuming all actions are 
available (this can be derived from the optimal value function, computed
using standard techniques). Then at each stage, we
use the best action available at the current state where actions
are ranked by Q-value. Unfortunately, this leads to a suboptimal policy
when the $\Up$ action has low availability, specifically if $p < 0.5$.

The best naive policy always chooses to move to $s_2$ from $s_1$; at
$s_2$, it picks the best action available. This yields a reward of $1/2$ at
even stages, and an expected reward of $p$ at odd stages. However, by 
anticipating the possibility that action $\Up$
is unavailable at $s_2$, the optimal
(SAS) policy always stays at $s_1$, obtaining reward $1/2$ at all stages. 
For $p < 1/2$, the latter policy dominates the former: the plot on the right
shows the fraction of the optimal (SAS) value \emph{lost} by the naive
policy ($Std$) as a function of the availability probability $p$.
This example also illustrates that as action availability probabilities
approach $1$, the optimal policy for the base MDP is also optimal for 
the SAS-MDP.

\subsection{Related Work}
\label{sec:related}

While a general formulation of MDPs with stochastic action availability
does not appear in the literature, there are two strands of closely
related work. In the bandits literature, \emph{sleeping bandits} are
defined as bandit problems in which the arms available at each stage
are determined randomly or adversarially (sleeping experts are similar, with
complete feedback being provided rather than bandit feedback)
\cite{kleinbergEtAl:MLJ2010,kanadeEtAl:aistats09Sleeping}.
Best action orderings (analogous to our decision list policies for SAS-MDPs)
are often used to define regret in these models. The goal is to develop
exploration policies to minimize regret. Since these models have no state,
if the action reward distributions are known, the optimal policy is
trivial: always take the best \emph{available} action. By contrast,
a SAS-MDP, even a known model, induces a
difficult optimization problem, since the quality of an action depends
not just on its immediate reward, but also on the availability of actions
at reachable (future) states. This is our focus.

The second closely related branch of research comes from the field of
stochastic routing. The ``Canadian Traveller Problem''---the
problem of minimizing travel time in a graph with unavailable edges---was
introduced by Papadimitriou and
Yannakakis~\cite{papadimitriou:shortestpath}, who gave intractability
results (under much weaker assumptions about edge availability, e.g.
adversarial). Poliyhondrous and Tsitsiklis~\cite{polychondrous:recourse}
consider a stochastic version of the problem, where edge availabilities are
random but static (and any edge observed to be unavailable remains so throughout
the scenario). Most similar to our setting is the work of Nikolova and
Karger~\cite{nikolova:canadian}, who discuss the case of resampling
edge costs at each node visit; however, the proposed solution is well-defined
only when the edge costs are finite and does not easily extend to unavailable
actions/infinite edge costs. Due to the specificity of their modeling
assumptions, none of the solutions found in this line of research can be
adapted in a straightforward way to SAS-MDPs.


\section{Two Reformulations of SAS-MDPs}
\label{sec:reformulate}

The randomness of feasible actions means that SAS-MDPs do not
conform to the usual definition of an MDP. In this section, we
develop two reformulations of SAS-MDPs that transform them into
MDPs. We discuss the relative advantages of each,
outline key properties and relationships
between these models, and describe important
special cases of the SAS-MDP model itself.

\subsection{The Embedded MDP}
\label{sec:embedded}

We first consider a reformulation of the SAS-MDP in which
we embed the (realized) available action set into the state space 
itself. This is a straightforward way to recover a standard MDP.
The \emph{embedded MDP} $\calM_e$
for a SAS-MDP
has state space $S_e = \{s\circ A: s\in S, A\subseteq B\}$, with
$s\circ A$ having feasible action set $A$.\footnote{Embedded states
whose embedded action subsets have zero probability are unreachable and
can be ignored.} The history
independence of $P_s$ allows transitions to be defined as:
$$p^k_{s\circ A,s'\circ A'} = 
   P(s'\circ A' | s\circ A, k) = p^k_{s,s'} P_{s'}(A')
, 
\;\;
\forall k\in A.$$
Rewards are defined similarly: $r^k(s\circ A) = r^k(s)$ 
for $k\in A$.

In our earlier example, the embedded MDP has three states:
$s_1\circ\{\Stay,\Go\}, s_2\circ\{\Up,\Down\}, s_2\circ\{\Down\}$
(other action subsets have probability $0$ hence their corresponding
embedded states are unreachable). The feasible actions at each state are
given by the embedded action set, and the only stochastic transition
occurs when $\Go$ is taken at $s_1$: it moves to 
$s_2\circ\{\Up,\Down\}$ with probability $p$ and $s_2\circ\{\Down\}$
with probability $1-p$.

Clearly, the induced reward process and dynamics are Markovian, hence 
$\calM_e$ is in fact an MDP under the usual definition. Given
the natural translation afforded by the embedded MDP, we
view this as providing the basic ``semantic'' underpinnings of
the SAS-MDP model. This translation affords the use of
standard MDP analytical tools and methods.

A (stationary, determinstic, Markovian) policy
$\pi:S_e \rightarrow B$ for $\calM_e$ is restricted so that 
$\pi(s\circ A) \in A$.
The policy backup operator $T^\pi_e$ and Bellman operator 
$T^\ast_e$ for $\calM_e$ decompose naturally as follows:
\begin{small}
\begin{align}
  &T^\pi_e V_e(s\circ A_s)
    = r^{\pi(s\circ A_s)}_s + \nonumber
     \\ &\qquad \gamma\sum_{s'}p^{\pi(s\circ A_s)}_{s,s'}
          \sum_{A_{s'}\subseteq B} P_{s'}(A_{s'})V_e(s'\circ A_{s'})
          \label{eq:embeddedpolicybackup},
\\
   &T^\ast_e V_e(s\circ A_s)
    = \max_{k\in A_s} r^k_s + \nonumber
    \\ &\qquad \gamma\sum_{s'}p^k_{s,s'}
          \sum_{A_{s'}\subseteq B} P_{s'}(A_{s'})V_e(s'\circ A_{s'})
          \label{eq:embeddedbellmanbackup}
\end{align}
\end{small}%
Their fixed points, $V^\pi_e$ and $V^\ast_e$ respectively, can be 
expressed similarly.
%
%

Obtaining an MDP from an SAS-MDP via action-set embedding comes at the
expense of a (generally) exponential blow-up in the size of the state
space, which can increase by a factor of $2^{|B|}$.

\subsection{The Compressed MDP}
\label{sec:compressed}

The embedded MDP provides a natural semantics for SAS-MDPs,
but is problematic from an algorithmic and learning perspective given
the state space blow-up.
Fortunately, 
the history independence of the availability distributions gives rise to 
an effective, compressed representation.
The \emph{compressed MDP} $\calM_c$
recasts the embedded MDP in terms of the original state space, 
using expectations to
express value functions, policies, and backups
over $S$ rather than over
the (exponentially larger) $S_e$. As we will see below, the compressed
MDP induces a blow-up in action space rather than state space, but
offers significant computational benefits.

Formally, the state space for $\calM_c$ is $S$. 
To capture action availability,
the feasible action set
for $s\in S$ is the set of \emph{state policies}, or
mappings
$\mu_s: 2^B \rightarrow B$ satisfying 
$\mu_s(A_s) \in A_s$.
In other words, once we reach $s$, $\mu_s$ dictates what action
to take for any realized action set $A_s$.
A policy for $\calM_c$ is a family $\mu_c = \{\mu_s :s\in S\}$ of such
state policies.  Transitions and rewards
use expectations over $A_s$:
\begin{small}
\begin{align*}
  p^{\mu_s}_{s,s'} 
     = \sum_{A_{s}\subseteq B} P_s(A_{s}) p^{\mu_s(A_s)}_{s,s'}
~~\mbox{and}~~
  r^{\mu_s}_{s} 
     = \sum_{A_{s}\subseteq B} P_s(A_{s}) r^{\mu_s(A_s)}_{s}~.
\end{align*}
\end{small}

In our earlier example, the compressed MDP has only two states,
$s_1$ and $s_2$. Focusing on $s_2$, its ``actions'' in the compressed
MDP are the set of state policies, or
mappings from the realizable available sets
$\{\{\Up,\Down\}, \{\Down\}\}$ into action choices (as above,
we ignore unrealizable action subsets): in this case, there are
two such state policies: 
the first selects $\Up$ for
$\{\Up,\Down\}$ and (obviously) $\Down$ for $\{\Down\}$;
the second selects $\Down$ for
$\{\Up,\Down\}$ and $\Down$ for $\{\Down\}$.

It is not hard to show that 
the dynamics and reward
process defined above
over this compressed state space
and expanded action set (i.e., the set of state policies) are
Markovian. Hence we can define policies, value functions, optimality
conditions, and policy and Bellman backup operators 
in the usual fashion. For instance,
the Bellman and policy backup operators, $T^\star_c$ and $T_{\mu}^c$,
on 
compressed value functions are:
\begin{small}
\begin{align}
T_c^*V_c(s) =& 
     \E_{A_s\subseteq B}\;
        \max_{k\in A_s} r^k_s  + \gamma \sum_{s'} p^k_{s,s'} V_c(s'),
        \label{eq:compressedBellmanOp}
        \\
T^{\mu}_cV_c(s) =& 
     \E_{A_s\subseteq B}\;
        r^{\mu_s(A_s)}_s + \gamma \sum_{s'} p^{\mu_s(A_s)}_{s,s'} V_c(s').
        \label{eq:compressedPolicyOp}
\end{align}
\end{small}

It is easy to see that any state policy $\mu$ induces a Markov chain over
base states, hence we can define a standard $n\times n$ transition matrix
$P^{\mu}$ for such a policy in the compressed MDP, where
$p^{\mu}_{s,s'} = \E_{A\subseteq B} p^{\mu(s)(A)}_{s,s'}$.
When additional independence assumptions hold,
this expectation over subsets can be computed efficiently (see \cref{sec:pda}). 

Critically, we can show that there is a direct ``equivalence''
between policies and their value functions (including optimal policies
and values) in $\calM_c$ and $\calM_e$.
Define the action-expectation operator 
$E: \mathbb{R}^{n2^m} \rightarrow \mathbb{R}^n$ to be
a mapping that compresses a value function $V_e$ for $\calM_e$ into
a value function
$V^e_c$ for $\calM_c$:
\begin{small}
$$
V^e_c(s) 
= EV_e(s) 
=\!\!\! \E_{A_s\subseteq B} V_e(s\circ A_s)
=\!\!\! \sum_{A_s \subseteq B} P_s(A_{s}) V_e(s\circ A_s).
$$
\end{small}%
We emphasize that $E$ transforms an (arbitrary) value function $V_e$
in embedded space
into a new value function $V_c^e$ defined
in compressed space (hence, $V_c^e$ is \emph{not} defined w.r.t.\
$\calM_c$).

\begin{lemma}
\label{lemma1}
$ET^*_e V_e = T_c^*EV_e$. Hence, $T^*_c$ has a unique fixed point $V_c^\ast = EV_e^\ast$.
\end{lemma}
\begin{proof}
\begin{small}
\begin{align*}
ET^eV_e(s) & = \E_{A\subseteq B} T^eV_e(s\circ A)\\
           & = \E_{A\subseteq B} 
                 \max_{k\in A} r^k_s +\gamma\sum_{s'\circ A'} 
                              p^k_{s\circ A,s'\circ A'} V_e(s'\circ A') \\
           & = \E_{A\subseteq B} 
                 \max_{k\in A} r^k_s +\gamma\sum_{s'} p^k_{s,s'}
                      \E_{A'\subseteq B} V_e(s'\circ A') \\
           & = \E_{A\subseteq B} 
                 \max_{k\in A} r^k_s +\gamma\sum_{s'} p^k_{s,s'} EV^e(s') \\
           & = T^cEV^e(s')
.
\end{align*}
\end{small}
\end{proof}

\begin{lemma}
\label{lemma2}
Given the optimal value function $V^\ast_c$ for $\calM_c$,
the optimal policy $\pi^\ast_e$ for $\calM_e$ can be constructed 
  directly. Specifically, for any $s\circ A$, the optimal policy
  $\pi^\ast_e(s\circ A)$ and optimal value $V^\ast_e(s\circ A)$ at that
  embedded state can be computed in polynomial time.
\end{lemma}
\begin{sketch}
Given $s\circ A$, the expected value of each action in $k\in A$ 
    can be computed using a one-step backup of $V^\ast_c$.
Then $\pi^\ast_e(s\circ A)$ is the action with maximum value,
and $V^\ast_e(s\circ A)$ is its backed-up expected value.
\end{sketch}

Therefore, it suffices to work directly with the compressed MDP, which
allows one to use value functions (and $Q$-functions) over the original 
state space.  The price is that one needs to use state policies,
since the best action at $s$ depends on the available set $A_s$.
In other words, while the embedded MDP causes an exponential blow-up in
state space, the compressed MDP causes an exponential blow-up in
action space.  We now turn to assumptions that allow us to
effectively manage this action space blow-up.


\subsection{Decision List Policies}
\label{sec:lists}

The embedded and compressed MDPs do not, \emph{prima facie},
offer much computational or representational advantage,
since they rely on an exponential increase in the size of the state
space (embedded MDP) or decision space (compressed MDP).
Fortunately, SAS-MDPs have optimal policies with a useful, concise form.
We first focus on the policy representation itself, 
then describe the considerable computational
leverage it provides.


A \emph{decision list (DL) policy} $\mu$ is a type of policy for
$\calM_e$ that can be expressed compactly using $O(nm \log m)$ space
and executed efficiently.
Let $\Sigma_B$ be the set of permutations over base action
set $B$. A DL policy $\mu: S \rightarrow \Sigma_B$ associates a permutation
$\mu(s) \in \Sigma_B$ with each state, and is executed at embedded
state $s\circ A$ by executing $\min \{i \in \{1,\ldots,m\} : \mu(s)(i) \in A\}$.
In other words,
whenever base state $s$ is encountered and $A$ is the available set, the
first action $k\in A$ in the order dictated by DL $\mu(s)$ is executed.
Equivalently, we can view $\mu(s)$ as a state policy $\mu_s$ for
$s$ in 
$\calM_c$.
In our earlier example, one DL $\mu(s_2)$ is $[\Up,\Down]$, which requires
taking (base) action $\Up$ if it is available, otherwise taking $\Down$.

For any SAS-MDP, we have optimal DL policies:
\begin{thm}
  $\calM_e$ has an optimal policy that can be represented
  using a decision list. The same policy is optimal for the 
  corresponding $\calM_c$.
\end{thm}
\begin{sketch}
Let $V^\ast$ be the (unique) optimal value function for $\calM_e$
    and $Q^\ast$ its corresponding Q-function (see Sec.~\ref{sec:valueiteralg}
    for a definition).
A simple inductive argument shows
that no DL policy is optimal
only if there is some state $s$, action sets $A \neq A'$, and
	(base) actions $j \neq k$, s.t.\ (i) $j,k \in A, A'$; (ii) for
    some optimal policy
    $\pi^\ast(s\circ A) = j$ and
    $\pi^\ast(s\circ A') = k$;
	and (iii) either $Q^\ast(s\circ A,j) > Q^\ast(s\circ A,k)$ or
    or $Q^\ast(s\circ A',k) > Q^\ast(s\circ A',j)$.
However, the fact that the optimal
Q-value of any action $k\in A$ at state $s\circ A$ is independent of the
other actions in $A$ (i.e., it depends only on the base state) implies that
these conditions are mutually contradictory.
\end{sketch}

\subsection{The Product Distribution Assumption}
\label{sec:pda}

The DL form
ensures that optimal policies and value functions
for SAS-MDPs can be expressed polynomially in the size
of the base MDP $\calM$.
However, their computation 
still requires the computation of
expectations over action subsets,
e.g., in Bellman or policy backups
(Eqs.~\ref{eq:compressedBellmanOp},~\ref{eq:compressedPolicyOp}).
This will generally be infeasible without
some assumptions on the form the action availability distributions $P_s$.

One natural assumption is the \emph{product distribution assumption (PDA)}.
PDA holds when $P_s(A)$ is a 
product distribution where each action
$k\in B$ is available with probability $\rho^k_s$, and subset
$A \subseteq B$ has probability 
$\rho_s^A = \prod_{k\in A} \rho^k_s \prod_{k\in B\setminus A} (1-\rho^k_s)$.
This assumption is a reasonable approximation
in the settings discussed above, 
where state-independent exogenous processes determine the availability of 
actions 
(e.g., the probability that one advertiser's campaign
has budget remaining is roughly
independent of another advertiser's).
For ease of notation, we
assume that $\rho^k_s$ is identical for all states $s$ 
(allowing different availability probabilities
across states has no impact on what follows). 
To ensure the MDP is well-founded, we assume some
default action (e.g., no-op) is always available.%
\footnote{We omit the default action
from analysis for ease of exposition.}
Our earlier running example trivially satisifes PDA: at $s_2$,
$\Up$'s availability probability ($p$) is independent of the
availability of $\Down$ (1).

When the PDA holds, the DL form of policies allows the expectations 
in policy and Bellman backups to be computed efficiently
without enumeration of
subsets $A\subseteq B$. For example, given a fixed DL policy $\mu$,
we have
\begin{small}
\begin{align}
T^{\mu}_cV_c(s) &= 
\sum_{i=1}^{m} \,\left[\prod_{j=1}^{i-1} (1-\rho^{\mu(s)(j)}_s)\right]\, \rho^{\mu(s)(i)}_s 
 \Bigg( r_s^{\mu(s)(i)} \nonumber \\
 &~ + \gamma \sum_{s'} p^{\mu(s)(i)}_{s,s'} V_c(s') \Bigg).
 \label{eq:pdabellman}
\end{align}
\end{small}%
The Bellman operator has a similar form.
We exploit this below to develop tractable value iteration and policy
iteration algorithms, as well as a practical LP formulation.

\subsection{Arbitrary Distributions with Sampling (ADS)}
\label{sec:ads}

We can also handle the case where, at each state, the availability
distribution is unknown, 
but is sampleable.
In the
longer version of the paper \cite{sasmdps_full:arxiv18},
we show that samples can be used to approximate expectations
w.r.t.\ available action subsets, and that the required sample 
size is polynomial in $|B|$, and not in the size of the \emph{support} of the
distribution. 

Of course, when we discuss algorithms for policy
computation, this approach does not allow us to compute the
optimal policy exactly. However, it has important implications for
sample complexity of learning algorithms like Q-learning. We note
that the ability to sample available action subsets is quite natural
in many domains. For instance, in ad domains, it may not be possible to
model the process by which eligible ads are generated (e.g.,
involving specific and evolving advertiser targeting criteria,
budgets, frequency capping, etc.). But the eligible subset of ads
considered for each impression opportunity
is an action-subset sampled from this process.

Under ADS, we compute approximate backup operators as follows.  Let
$\mathcal{A}_s = \{A_s^{(1)},\ldots,A_s^{(T)}\}$ be an i.i.d. sample of size
$T$ of action subsets in state $s$. For a subset of actions $A$, an index $i$
and a decision list $\mu$, define 
$I_{[i,A,\mu]}$ to be 1 if $\mu(i) \in A$
and for each $j < i$ we have $\mu(j) \not\in A$, or 0 otherwise.
Similar to \cref{eq:pdabellman}, we define:
{
\small
\begin{align*}
T^{\mu}_cV_c(s) &\!=\! \frac{1}{T} \sum_{t=1}^T 
\sum_{i=1}^{m} I_{\left[i,A_s^{(t)}, \mu(s)\right]}
\!
 \Big( r_s^{\mu(s)(i)} \!+\! \gamma \!\sum_{s'} p^{\mu(s)(i)}_{s,s'} V_c(s') \Big).
\end{align*}
}%
In the sequel, we focus largely on PDA; 
in most cases equivalent results can be derived in the ADS model.

\section{Q-Learning with the Compressed MDP}
\label{sec:qlearn}

Suppose we are faced with learning the optimal value function or
policy for an SAS-MDP from a collection of trajectories. 
The (implicit)
learning of the transition dynamics and rewards
can proceed as usual;
the novel aspect of the SAS model is that the action availability distribution
must also be considered.
Remarkably, Q-learning can be readily augmented to incorporate
stochastic action sets: we require only that our training trajectories
are augmented with the set of actions that were available at
each state,
$$\ldots s^{(t)}, A^{(t)}, k^{(t)}, r^{(t)},
         s^{(t+1)}, A^{(t+1)}, k^{(t+1)}, r^{(t+1)}, \ldots,$$
where:
$s^{(t)}$ is the realized state at time $t$ (drawn from distribution
$P(\cdot|s^{(t-1)}, k^{(t-1)})$);
$A^{(t)}$ is the realized available set at time $t$, drawn from
$P_{s^{(t)}}$; $k^{(t)}\in A^{(t)}$ is the action taken;
and $r^{(t)}$ is the realized reward.
Such augmented trajectory data is typically available. In particular,
the required sampling of available action sets is usually feasible
(e.g., in ad serving as discussed above).
%

\emph{SAS-Q-learning} can be applied directly
to the compressed MDP $\calM_c$, requiring only a
minor modification of the standard Q-learning update for the base 
MDP. We simply require that each Q-update maximize over
the \emph{realized available actions} $A^{(t+1)}$:
\begin{small}
\begin{align*}
Q^{\new}(s^{(t)},k^{(t)}) &\leftarrow 
	(1-\alpha_t) Q^{\old}(s^{(t)},k^{(t)}) \\ 
    &\quad +
    \alpha_t [r^{(t)} + \gamma \max_{k\in A^{(t+1)}} Q^{\old}(s^{(t+1)},k)]~.
\end{align*}
\end{small}%
Here $Q^{\old}$ is the previous $Q$-function estimate and 
$Q^{\new}$ is the updated estimate, thus it encompasses both
online and batch Q-learning, experience replay, etc.; and
$0\leq \alpha_t < 1$ is our (adaptive) learning rate.

It is straightforward to show that, under the usual exploration conditions,
SAS-Q-learning will converge to the optimal Q-function for the compressed MDP,
since the expected maximum over sampled 
action sets at any particular state will converge to the expected maximum 
at that state.
\begin{thm}
The SAS-Q-learning algorithm will converge w.p.\ 1
to the optimal Q-function for the 
(discounted, infinite-horizon) compressed MDP $\calM_c$ if the
usual stochastic approximation requirements are satisfied. That is, if
(a) rewards are bounded and (b) the subsequence of learning rates
$\alpha_{t(s,k)}$ applied to $(s,k)$ satisfies
$\sum \alpha_{t(s,k)} = \infty$ and
$\sum \alpha^2_{t(s,k)} < \infty$ for all state-action pairs $(s,k)$
	(see, e.g., \cite{watkins:mlj92}).
\end{thm}
Moreover, function approximation techniques,
such as DQN \cite{mnih2015},
can be directly applied with the same action set-sample maximization.
Implementing an optimal policy is also straightforward: given a state $s$
and the realization $A_s$ of the available actions, one simply executes
$\arg\max_{k\in A_s} Q(s,k)$.

We note that extracting the optimal value function $V_c(s)$
for the compressed MDP from the learned Q-function is not viable without
some information about the action availability distribution.
Fortunately, one need not know
the expected value at a state to implement
the optimal policy.\footnote{It is, of course, straightforward to learn 
an optimal value function if desired.}

\section{Value Iteration in the Compressed MDP}
\label{sec:valueiter}

Computing a value function for 
$\calM_c$, with its ``small'' state space $S$,
suffices to execute an optimal policy.
We develop an efficient \emph{value iteration (VI)} method to do this.

\subsection{Value Iteration}
\label{sec:valueiteralg}

Solving an SAS-MDP using VI
is challenging in general due to the required expectations over
action sets. 
However, under PDA, we can derive an
efficient VI algorithm whose complexity depends only polynomially
on the base set size $|B|$.
%

Assume a current iterate $V^t$, where
$
V^t(s) = \E_{A_s} [\max_{k\in A_s} Q^t(s,k) ]
$.
We compute $V^{t+1}$ as follows:
\begin{itemize}\denselist
  \item For each $s\in S, k\in B$, compute its $(t+1)$-stage-to-go Q-value:
    $Q^{t+1}(s,k) = r^k_s + \gamma \sum_{s'} p^k_{s,s'} V^t(s').$
  \item Sort these Q-values in descending order.
For convenience, we re-index each action by 
    its Q-value rank (i.e., $k_{(1)}$ is the action with largest Q-value, and $\rho_{(1)}$ is its probability,  
    $k_{(2)}$ the second-largest, etc.).
  \item For each $s\in S$, compute its $(t+1)$-stage-to-go value:
\begin{small}
\begin{align*}
 V^{t+1}(s) & = \E\nolimits_{A_s} \left[\max_{k\in A_s} Q^{t+1}(s,k)\right]\\
 & = \sum_{i=1}^{m-1} \left( \prod_{j=1}^{i-1} (1 - \rho_{(j)}) \right) \rho_{(i)} Q^{t+1}(s,k_{(i)})
.
 \end{align*} 
\end{small}
\end{itemize}

Under ADS, we use the approximate Bellman operator:
\begin{small}
\begin{align*}
\widehat{V}^{t+1}(s) &= \E\nolimits_{A_s} \left[\max_{k\in A_s} \widehat{Q}^{t+1}(s,k)\right]
\\
&= \frac{1}{T} \sum_{t=1}^T 
\sum_{i=1}^{m} I_{\left[i,A_s^{(t)}, \mu(s)\right]} \widehat{Q}^{t+1}(s,\mu(s)(i))~,
 \end{align*} 
\end{small}%
where $\mu(s)$ is the DL resulting from sorting $\widehat{Q}^{t+1}$-values.

The Bellman operator under PDA is tractable:
\begin{obs} 
    \label{obs:VIperiteration}
    The compressed Bellman operator $T^*_c$ 
    can be computed in $O(n m\log m)$ time.
\end{obs}
Therefore the per-iteration time complexity of VI for $\calM_c$
compares favorably to the $O(n m)$ time of VI in the base MDP.
The added complexity arises from the need to sort Q-values.%
\footnote{The products of the action availability probabilities
can be computed in linear time via caching.} 
Conveniently, this sorting process immediately
provides the desired DL state policy for $s$.

Using standard arguments, we obtain the following
results, which immediately yield a polytime approximation method.
\begin{lemma}
\label{lemma3}
$T^*_c$ is a contraction with modulus $\gamma$
i.e., 
$||T^*_c v_c - T^*_c v'_c|| \leq \gamma ||v_c - v'_c||$.
\end{lemma}

\begin{cor}
For any precision $\veps < 1$, the compressed value iteration 
algorithm converges to an $\veps$-approximation of the optimal value
function in $O(\log (L/\veps))$ iterations, where 
$L\leq [\max_{s,k} r^k_s ]/(1-\gamma)$
is an upper bound on $||V^\ast_e||$.
\end{cor}
We provide an even stronger result next: VI, in fact,
converges to an \emph{optimal} solution in polynomial time.

\subsection{The Complexity of Value Iteration}
\label{sec:valueitercomplexity}

Given its polytime per-iteration complexity, to ensure VI 
is polytime, we must show that it converges to a value function
that induces an optimal policy in polynomially many iterations.
To do so, we exploit the compressed representation and adapt
the technique of \cite{tseng:ORLetters90}.

%

Assume, w.r.t.~the base MDP $\calM$, that the discount factor $\gamma$, rewards
$r^k_s$, and transition probabilities $p^k_{s,s'}$, are rational numbers
represented with a precision of $1/ \delta$ ($\delta$ is an integer).
Tseng shows that VI for a standard MDP is strongly polynomial, assuming
constant $\gamma$ and $\delta$, by proving that: (a) if the $t$'th value
function produced by VI satisfies
$$
||V^t - V^\ast|| < 
1/(2 \delta^{2n+2} n^n)
,
$$
then the policy induced by $V^t$ is optimal;
and (b) VI achieves this bound in polynomially many
iterations.

We derive a similar bound on the number of VI iterations
needed for convergence in an SAS-MDP,
using the same input parameters as in the base MDP, and applying
the same precision $\delta$ to the action availability probabilities.
We apply Tseng's result by exploiting the fact that: (a) the optimal
policy for the embedded MDP $\calM_e$ can be represented as a DL; (b) the
transition function for any DL policy can be expressed using an $n\times n$
matrix (we simply take expectations, see above); and (c) the
corresponding linear system can be expressed over the \emph{compressed}
rather than the embedded state space 
to determine $V_c^\ast$ (rather than $V_e^\ast$). 

Tseng's argument requires some adaptation to apply to
the compressed VI algorithm. 
We extend his precision assumption to account for our action availability
probabilities as well, ensuring $\rho^k_s$ is also represented up to precision of $1/\delta$. 

Since $\calM_c$ is an MDP, Tseng's result applies; but notice that
each entry of the transition matrix for any state's DL $\mu$, which serves
as an action in $\calM_c$, is a product of $m+1$ probabilities, each with precision
$1/\delta$. We have that $p^\mu_{s,s'}$ has precision of $1/\delta^{m+1}$.
Thus the required precision parameter for our MDP is at most
$\delta^{m+1}$. Plugging this into Tseng's bound, 
VI applied to $\calM_c$ must induce an optimal policy at the $t$'th iteration if
{\small
$$
||V^t - v^\ast|| < 1/(2({\delta^{(m+1)}})^{2n} n^n)
                   = 1/(2\delta^{(m+1)2n} n^n)~.
$$
}%
This in turn gives us a bound on the number of iterations of VI
needed to reach an optimal policy:
\begin{thm} 
    \label{thm:VIconvergence}
VI applied to $\calM_c$ converges to a value function whose greedy policy 
is optimal in $t^*$ iterations, where
{\small
$$
t^* \leq \log(2\delta^{2n(m+1)} n^n M) / \log(1/\gamma)
$$
}
\end{thm}
Combined with Obs.~\ref{obs:VIperiteration}, we have:
\begin{cor} 
    \label{cor:VIpolytime}
VI yields an optimal policy for the SAS-MDP corresponding to $\calM_c$
in polynomial time.
\end{cor}

Under ADS, VI merely approximates
the optimal policy. In fact, one cannot compute an exact
optimal policy without observing the entire support of the availability
distributions (requiring exponential sample size).

\section{Policy Iteration in the Compressed MDP}
\label{sec:policyiter}

We now outline a policy iteration (PI) algorithm.

\subsection{Policy Iteration}
\label{sec:policyiteralg}

The concise DL form of optimal policies can be exploited in
PI as well.
Indeed, \emph{the greedy policy $\pi^V$ with respect to any value function $V$
in the compressed space} is representable as a DL. 
Thus the policy improvement
step of PI can be executed using the same independent 
evaluation of action Q-values and sorting as used in VI above:
\begin{small}
\begin{gather*}
Q^V(s,k) = r(s,k) + \gamma \sum_{s'} p^k_{s, s'} V(s')
,\\
Q^V(s,A_s) \!=\! \max_{k\in A_s} Q^V(s,k) 
    \;\textrm{, and } \; 
\pi^V(s,A_s) \!=\! \arg\max_{k\in A_s} Q^V(s,k).
\end{gather*}
\end{small}

The DL policy form can also be exploited in the policy evaluation 
phase of PI.
The tractability of policy evaluation requires a tractable representation
of the action availability probabilities, which
PDA provides, leading to the following PI method that exploits PDA:
\begin{enumerate}\denselist
  \item Initialize an arbitrary policy $\pi$ in decision list form.
  \item Evaluate $\pi$ by solving the following linear system over variables
    $V^\pi(s), \forall s\in S$: (Note: We use $Q^{\pi}(s,k)$ to represent the
    relevant linear expression over $V^\pi$.)
    \begin{small}
    \begin{align*}
      V^\pi(s) 
         &= \sum_{i=1}^{n} \,[\prod_{j=1}^{i-1} (1-\rho_{(j)})]\, \rho_{(i)} Q^{\pi}(s,k_{(i)})
         \label{eq:pda_expect_val}
    \end{align*}
    \end{small}
  \item Let $\pi'$ denote the greedy policy w.r.t.\ $V^\pi$, which can be 
    expressed in DL form for each $s$ by sorting Q-values 
    $Q^{\pi}(s,k)$ as above
    (with standard tie-breaking rules).
    If $\pi'(s) = \pi(s)$, terminate;
    otherwise replace $\pi$ with $\pi'$ and repeat (Steps 2 and 3).
\end{enumerate}

Under ADS, PI can use the approximate Bellman operator, giving
an approximately optimal policy.

\subsection{The Complexity of Policy Iteration}
\label{sec:policyitercomplexity}

The per-iteration complexity of PI in $\calM_c$ is 
polynomial: as in standard PI, policy evaluation solves
an $n\times n$ linear system (naively, $O(n^3)$) plus
the additional overhead (linear in $M$) to compute the compounded
availability probabilities; 
and policy improvement requires $O(mn^2)$ computation of action Q-values, 
plus $O(nm\log m)$ overhead for sorting Q-values (to produce
improving DLs for all states).

An optimal policy is reached in a number of iterations
no greater than that required by VI, since:
(a) the sequence of value functions for the policies 
generated by PI contracts at least as quickly as the value
functions generated by VI (see, e.g., 
\cite{meister:ORSpektrum86,hansen:jacm13}); (b) our precision argument
for VI ensures that the greedy policy extracted at that
point will be optimal; and (c) once PI finds an optimal policy,
it will terminate (with one extra iteration). Hence, PI
is polytime (assuming a fixed discount $\gamma<1$).
\begin{thm} 
    \label{thm:PIpolytime}
    PI yields
an optimal policy for
    the 
SAS-MDP 
corresponding to $\calM_c$
in polynomial time.
\end{thm}
In the longer version of the paper \cite{sasmdps_full:arxiv18},
we adapt more direct proof techniques
\scite{yinyuye:mathOR11,hansen:jacm13} to derive polynomial-time
convergence of PI for SAS-MDPs under additional assumptions.
Concretely, for a policy $\mu$ and actions $k_1,k_2$,
let $\eta_\mu(s,k_1,k_2)$ be the probability, over action sets,
that at state $s$, the optimal 
$\mu^\star$ selects $k_1$ and
$\mu$ selects $k_2$.
Let $q > 0$ be such that 
$\eta_\mu(s,k_1,k_2) \ge q$
whenever
$\eta_\mu(s,k_1,k_2) > 0$.
We show:
\begin{thm}
\label{thm:policycomp}
The number of iterations it takes policy iteration to converge is no more than
{
\small
\[
O\left( \frac{nm^2}{1-\gamma} \log \frac{m}{1-\gamma} \log \frac{e}{q} \right)~.
\]
}
\end{thm}
Under PDA, the theorem implies \emph{strongly-polynomial} convergence of PI
if each action is available with constant probability. In this case, for
any $\mu$, $k_i$, $k_j$, and $s$, we have $\eta_\mu(s,k_i,k_j) \ge \rho_s^{k_i} \cdot
\rho_s^{k_j} = \Omega(1)$, which in turn implies that we can take $q = \Omega(1)$ in the bound above.


%
%

\section{Linear Programming in the Compressed MDP}
\label{sec:lp}

An alternative model-based approach is linear programming
(LP). The primal formulation for the embedded MDP $\calM_e$ is straightforward
(since it is a standard MDP), but requires exponentially many variables
(one per embedded state) and constraints (one per embedded state, base action
pair).
%

A (nonlinear)
primal formulation for the compressed MDP $\calM_c$ reduces the
number of variables to $|S|$:
\begin{small}
\begin{align}
  \min_{\mathbf{v}}\,\sum_{s\in S}\nolimits\alpha_s v_s,\quad\textrm{s.t. }\,v_s \geq \E\nolimits_{A_s}\max_{k\in A_s} Q(s,k)\quad \forall s.
      \label{eq:compressedPrimalConstr}
\end{align}
\end{small}%
Here $\alpha$ is an arbitrary, positive state-weighting,
over the embedded states corresponding to each base state
and 
$$Q(s,k) = r_s^k + \sum_{s' \in S} p^k_{s,s'} v_{s'}$$
abbreviates the linear expression of the action-value backup
at the state and action in question 
w.r.t.\ the value variables $v_s$.
This program is valid given the definition of $\calM_c$ and
the fact that a weighting over embedded states corresponds 
to a weighting over base states by taking expectations.
Unfortunately, this formulation is non-linear,
due to the max term in each constraint.
And while it has only
$|S|$ variables, it has factorially many constraints;
moreover, the constraints themselves are not compact
due to the presence of the expectation in each constraint.

PDA can be used to render this formulation tractable.
Let $\sigma$ denote an arbitrary (inverse) permutation of the action set 
(so $\sigma(i)=j$
means that action $j$ is ranked in position $i$). As above,
the optimal policy at base state $s$ w.r.t.\ a
Q-function is expressible as a DL
( with actions sorted by
Q-values) and its expected value given by the expression derived below. 
Specifically, if
$\sigma$ reflects the relative ranking of the (optimal) Q-values of the 
actions at some
fixed state $s$, then $V(s) = Q(s,\sigma(1))$
with probability $\rho_{\sigma(1)}$, i.e., the
probability that $\sigma(1)$ occurs in $A_s$. Similarly, 
$V(s) = Q(s,\sigma(2))$ with probability
$(1-\rho_{\sigma(1)})\rho_{\sigma(2)}$, and so on.
We define the Q-value of a DL $\sigma$ as follows:
\begin{small}
\begin{align}
Q^V_s(\sigma) = \sum_{i=1}^{n} 
 \,[\prod_{j=1}^{i-1} (1-\rho_{\sigma(j)})]\, 
    \rho_{\sigma(i)} Q^V(s,\sigma(i)).
\end{align}
\end{small}%
Thus, for any fixed action permutation $\sigma$,
the constraint that
$v_s$ at least matches the expectation 
of the maximum action's Q-value is linear.
Hence, the program can be recast as an LP by enumerating
action permutations for each base state, replacing the constraints in
 Eq.~(\ref{eq:compressedPrimalConstr}) as follows:
\begin{small}
\begin{align}
   v_s \geq Q^V_s(\sigma) \quad \forall s\in S, \forall \sigma \in \Sigma.
      \label{eq:compressedPrimalConstr2}
\end{align}
\end{small}
%

The constraints in this LP are now each compactly represented,
but it still has factorially many constraints.
Despite this, it can be solved in polynomial time.
First, we observe that the LP is well-suited to constraint generation.
Given a relaxed LP with a subset of constraints, 
a greedy algorithm that simply sorts actions by Q-value 
to form a permutation can be used to find the maximally violated 
constraint at any state. Thus we have a practical constraint generation
algorithm for this LP since (maximally)
violated constraints can be found in polynomial time.

More importantly from a theoretical standpoint,
%
the constraint generation algorithm can be used
as a separation oracle within an ellipsoid method for this LP.
This directly yields an exact, (weakly) 
polynomial time algorithm for
this LP~\cite{GroetschelLovaszSchrijver1988}.


\section{Empirical Illustration}
\label{sec:empirical}

We now provide a somewhat more elaborate empirical demonstration of the 
effects of stochastic action availability.
Consider an MDP that corresponds to a routing problem on a real-world road 
network (Fig.~\ref{fig:routing}) in the San Francisco Bay Area.
The shortest path between the source and destination locations is sought.
The dashed edge in Fig.~\ref{fig:routing} represents a bridge,
available only with probability $p$, while all other edges correspond to
action choices available with probability $0.5$.
At each node, a no-op action (waiting) is available at constant cost;
otherwise the edge costs are the geodesic lengths of the corresponding roads
on the map. 
The optimal policies for different choices $p=0.1, 0.2$ and $0.4$ are
depicted in Fig.~\ref{fig:routing}, 
where line thickness and color 
indicate traversal probabilities under the corresponding optimal policies.
It can be observed that lower values of $p$ lead to policies with more 
redundancy. 
Fig.~\ref{fig:obliv} investigates the effect of solving the routing problem obliviously to
the stochastic action availability (assuming actions are fully available). The SAS-optimal
policy allows graceful scaling of the expected travel time from source to destination as
bridge availability decreases. 
Finalluy, the effects of violating the PDA assumption are investigated in the long version
of this work~\cite{TODO}.
\begin{figure}[t!]
    \centering
        \includegraphics[width=0.5\columnwidth]{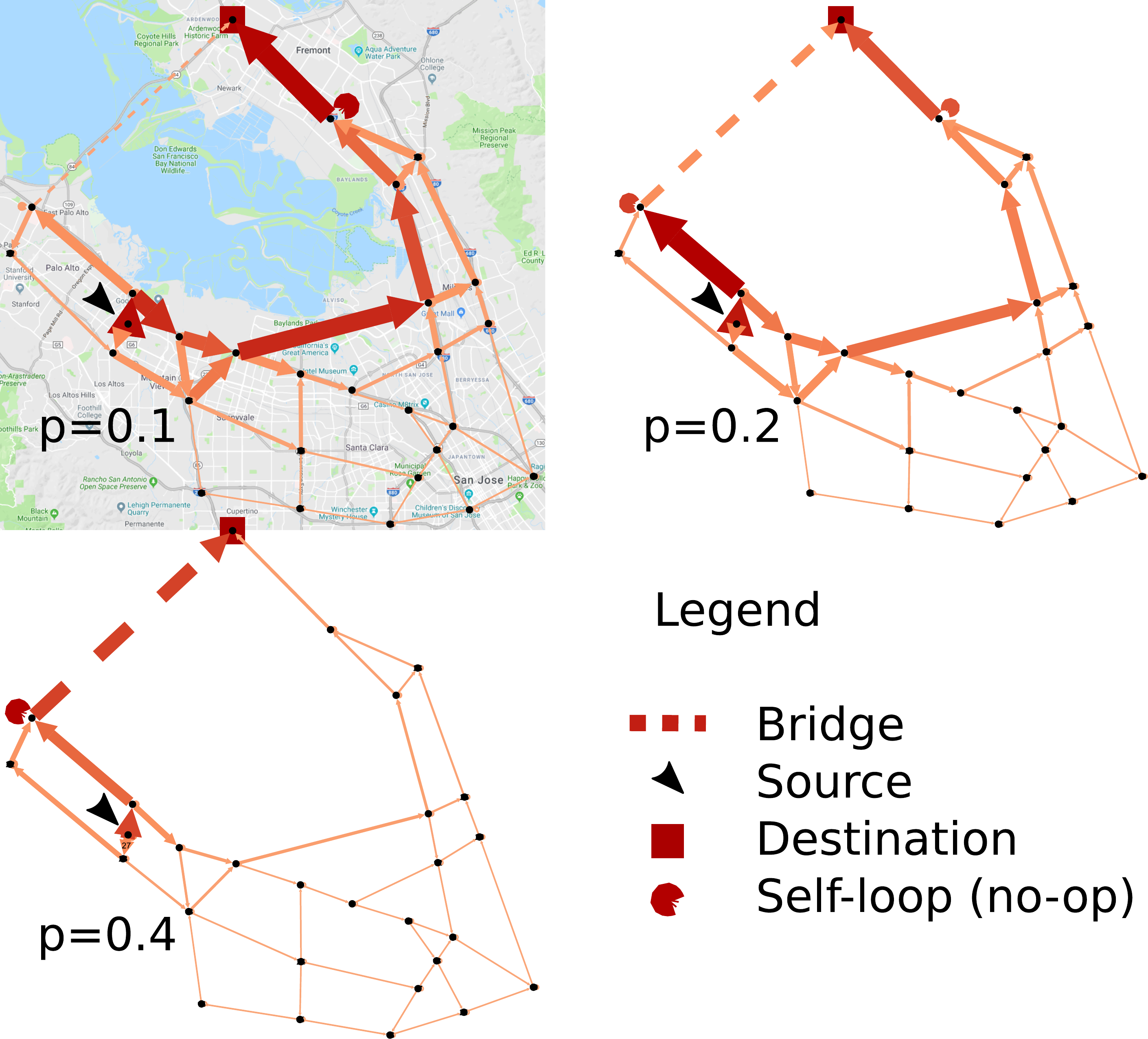}
    \caption{Stochastic action MDPs applied to routing.\label{fig:routing}\vspace{-5pt}}
\end{figure}
\vspace{-5pt}


\begin{figure}
  \begin{minipage}[c]{0.35\columnwidth}
    \includegraphics[width=\columnwidth]{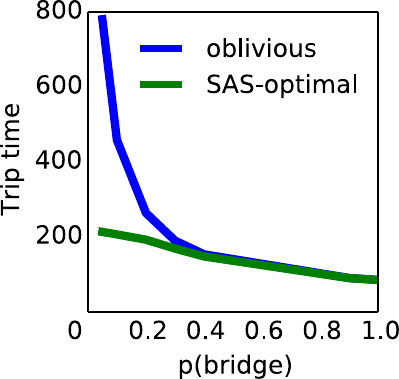}
  \end{minipage}\hfill
  \begin{minipage}[c]{0.6\columnwidth}
    \caption{
    Expected trip time from source to destination under the SAS-optimal policy vs. under the oblivious optimal policy (the MDP solved as if actions are fully available) as a function of bridge availability.}\label{fig:obliv}   
  \end{minipage}
\end{figure}

\vspace{-5pt}
\section{Concluding Remarks}
\label{sec:conclude}

We have developed a new MDP model, \emph{SAS-MDPs}, that extends the usual
finite-action MDP model by allowing the set of available actions
to vary stochastically.
This captures an
important use case that arises in many practical applications
(e.g., online advertising, recommender systems).
We have shown that embedding action sets in the state gives
a standard MDP, supporting tractable analysis at
the cost of an exponential blow-up in state space size. Despite this,
we demonstrated that (optimal and greedy) policies
have a useful decision list structure. We showed how
this DL format can be exploited
to construct
tractable Q-learning, value and policy iteration, and linear programming
algorithms.

While our work offers firm foundations for stochastic action
sets, most practical applications will not
use the algorithms described here explicitly. For example, in RL, we
generally use function approximators for generalization and scalability in
large state/action problems. We have successfully applied Q-learning 
using DNN function approximators (i.e.,~DQN)
using sampled/logged available actions in ads and recommendations
domains as described in Sec.~\ref{sec:qlearn}. This has allowed us
to apply SAS-Q-learning
to problems of significant, commercially viable scale.
Model-based methods such as VI, PI, and LP also require suitable
(e.g., factored) representations of MDPs and
structured implementations of our algorithms
that exploit these representations. For instance,
extensions of approximate linear programming or structured dynamic 
programming to incorporate stochastic action sets would be extremely
valuable.

Other important questions include
developing a polynomial-\emph{sized} direct LP formulation; and deriving
sample-complexity results for RL algorithms
like Q-learning is also of particular interest, especially as it pertains 
to the sampling of the action distribution. Finally, we are quite interested
in relaxing the strong assumptions embodied in the PDA model---of particular
interest is the extension of our algorithms to less extreme forms of
action availability independence, for example, as represented using consise
graphical models (e.g., Bayes nets).

\vskip 2mm
\noindent
\textbf{Acknowledgments:} Thanks to the reviewers for their helpful
suggestions.

\bibliographystyle{plain}
\bibliography{long,standard}

\begin{thebibliography}{10}

\bibitem{papadimitriou:shortestpath}
Shortest paths without a map.
\newblock {\em Theoretical Computer Science}, 84(1):127 -- 150, 1991.

\bibitem{kearns:uai12}
Kareem Amin, Michael Kearns, Peter Key, and Anton Schwaighofer.
\newblock Budget optimization for sponsored search: Censored learning in
  {MDPs}.
\newblock In {\em Proceedings of the Twenty-eighth Conference on Uncertainty in
  Artificial Intelligence (UAI-12)}, pages 543--553, Catalina, CA, 2012.

\bibitem{mirrokni:wine12}
Nikolay Archak, Vahab Mirrokni, and S.~Muthukrishnan.
\newblock Budget optimization for online campaigns with positive carryover
  effects.
\newblock In {\em Proceedings of the Eighth International Workshop on Internet
  and Network Economics (WINE-12)}, pages 86--99, Liverpool, 2012.

\bibitem{archak-mirrokni-muthu:www10}
Nikolay Archak, Vahab~S. Mirrokni, and S.~Muthukrishnan.
\newblock Mining advertiser-specific user behavior using adfactors.
\newblock In {\em Proceedings of the Nineteenth International World Wide Web
  Conference (WWW 2010)}, pages 31--40, Raleigh, NC, 2010.

\bibitem{sasmdps_full:arxiv18}
Craig Boutilier, Alon Cohen, Avinatan Hassidim, Yishay Mansour, Ofer Meshi,
  Martin Mladenov, and Dale Schuurmans.
\newblock Planning and learning in {Markov} decision processes with stochastic
  action sets.
\newblock {\tt arXiv:1805.02363}, 2018.

\bibitem{charikar:stoc99}
Moses Charikar, Ravi Kumar, Prabhakar Raghavan, Sridhar Rajagopalan, and Andrew
  Tomkins.
\newblock On targeting {Markov} segments.
\newblock In {\em Proceedings of the 31st Annual {ACM} Symposium on Theory of
  Computing (STOC-99)}, pages 99--108, Atlanta, 1999.

\bibitem{GroetschelLovaszSchrijver1988}
Martin Gr{\"o}tschel, L{\'a}szlo Lov{\'a}sz, and Alexander Schrijver.
\newblock {\em {Geometric Algorithms and Combinatorial Optimization}}, volume~2
  of {\em Algorithms and Combinatorics}.
\newblock Springer, 1988.

\bibitem{hansen:jacm13}
Thomas~Dueholm Hansen, Peter~Bro Miltersen, and Uri Zwick.
\newblock Strategy iteration is strongly polynomial for 2-player turn-based
  stochastic games with a constant discount factor.
\newblock {\em Journal of the ACM (JACM)}, 60(1), 2013.
\newblock Article 1, 16pp.

\bibitem{kanadeEtAl:aistats09Sleeping}
Varun Kanade, H~Brendan McMahan, and Brent Bryan.
\newblock Sleeping experts and bandits with stochastic action availability and
  adversarial rewards.
\newblock In {\em Twelfth International Conference on Artificial Intelligence
  and Statistics (AIStats-09)}, pages 272--279, Clearwater Beach, FL, 2009.

\bibitem{kleinbergEtAl:MLJ2010}
Robert Kleinberg, Alexandru Niculescu-Mizil, and Yogeshwer Sharma.
\newblock Regret bounds for sleeping experts and bandits.
\newblock {\em Machine learning}, 80(2--3):245--272, 2010.

\bibitem{Li:adkdd2009}
Ting Li, Ning Liu, Jun Yan, Gang Wang, Fengshan Bai, and Zheng Chen.
\newblock A {Markov} chain model for integrating behavioral targeting into
  contextual advertising.
\newblock In {\em Proceedings of the Third International Workshop on Data
  Mining and Audience Intelligence for Advertising (ADKDD-09)}, pages 1--9,
  Paris, 2009.

\bibitem{meister:ORSpektrum86}
U.~Meister and U.~Holzbaur.
\newblock A polynomial time bound for {Howard's} policy improvement algorithm.
\newblock {\em OR Spektrum}, 8:37--40, 1986.

\bibitem{logisticMDPs:ijcai17}
Martin Mladenov, Craig Boutilier, Dale Schuurmans, Ofer Meshi, Gal Elidan, and
  Tyler Lu.
\newblock Logistic {Markov} decision processes.
\newblock In {\em Proceedings of the Twenty-sixth International Joint
  Conference on Artificial Intelligence (IJCAI-17)}, pages 2486--2493,
  Melbourne, 2017.

\bibitem{mnih2013}
Volodymyr Mnih, Koray Kavukcuoglu, David Silver, Alex Graves, Ioannis
  Antonoglou, Daan Wierstra, and Martin Riedmiller.
\newblock Playing {Atari} with deep reinforcement learning.
\newblock {\em arXiv preprint arXiv:1312.5602}, 2013.

\bibitem{mnih2015}
Volodymyr Mnih, Koray Kavukcuoglu, David Silver, Andrei~A Rusu, Joel Veness,
  Marc~G Bellemare, Alex Graves, Martin Riedmiller, Andreas~K Fidjeland, Georg
  Ostrovski, et~al.
\newblock Human-level control through deep reinforcement learning.
\newblock {\em Nature}, 518(7540):529--533, 2015.

\bibitem{nikolova:canadian}
Evdokia Nikolova and David~R. Karger.
\newblock Route planning under uncertainty: The {Canadian} traveller problem.
\newblock In {\em Proceedings of the 23rd National Conference on Artificial
  Intelligence - Volume 2}, AAAI'08, pages 969--974. AAAI Press, 2008.

\bibitem{polychondrous:recourse}
George~H. Polychronopoulos and John~N. Tsitsiklis.
\newblock Stochastic shortest path problems with recourse.
\newblock {\em Networks}, 27(2).

\bibitem{puterman}
Martin~L. Puterman.
\newblock {\em {Markov} Decision Processes: Discrete Stochastic Dynamic
  Programming}.
\newblock Wiley, New York, 1994.

\bibitem{silver:icml13}
David Silver, Leonard Newnham, David Barker, Suzanne Weller, and Jason McFall.
\newblock Concurrent reinforcement learning from customer interactions.
\newblock In {\em Proceedings of the 30th International Conference on Machine
  Learning (ICML-13)}, pages 924--932, Atlanta, 2013.

\bibitem{theocharous:ijcai15}
Georgios Theocharous, Philip~S. Thomas, and Mohammad Ghavamzadeh.
\newblock Personalized ad recommendation systems for life-time value
  optimization with guarantees.
\newblock In {\em Proceedings of the Twenty-fourth International Joint
  Conference on Artificial Intelligence (IJCAI-15)}, pages 1806--1812, Buenos
  Aires, 2015.

\bibitem{tseng:ORLetters90}
Paul Tseng.
\newblock Solving h-horizon, stationary {Markov} decision problems in time
  proportional to log(h).
\newblock {\em Operations Research Letters}, 9(5):287--297, 1990.

\bibitem{watkins:mlj92}
Christopher J. C.~H. Watkins and Peter Dayan.
\newblock {Q}-learning.
\newblock {\em Machine Learning}, 8:279--292, 1992.

\bibitem{yinyuye:mathOR11}
Yinyu Ye.
\newblock The simplex and policy-iteration methods are strongly polynomial for
  the {Markov} decision problem with a fixed discount rate.
\newblock {\em Mathematics of Operations Research}, 36(4):593--603, 2011.

\end{thebibliography}



\appendix

\onecolumn

\section{Proofs}

\subsection{Proof of \cref{lemma1}}

\begin{align*}
ET^eV_e(s) & = \E_{A\subseteq B} T^eV_e(s\circ A)\\
           & = \E_{A\subseteq B} 
                 \max_{k\in A} r^k_s +\gamma\sum_{s'\circ A'} 
                              p^k_{s\circ A,s'\circ A'} V_e(s'\circ A') \\
           & = \E_{A\subseteq B} 
                 \max_{k\in A} r^k_s +\gamma\sum_{s'} p^k_{s,s'}
                      \E_{A'\subseteq B} V_e(s'\circ A') \\
           & = \E_{A\subseteq B} 
                 \max_{k\in A} r^k_s +\gamma\sum_{s'} p^k_{s,s'} EV^e(s') \\
           & = T^cEV^e(s')
.
\end{align*}

\subsection{Proof of \cref{lemma2}}

It is easy to compute the expected value of each action in $k\in A$ using a
one-step backup of $V^\ast_c$.
The action with maximum value is $\pi^\ast_e(s\circ A)$ and its backed-up 
expected value is $V^\ast_e(s\circ A)$.

\subsection{Proof of \cref{lemma3}}

Since $\calM_e$ is a standard MDP, we know $\calM_e$ is a contraction with
modulus $\gamma$.
Let $v_c = Ev_e$ and $v_c' = Ev'_e$ be compressed value functions 
that each correspond to some (arbitrary) embedded value function.
Then for any $s$ we have:
\begin{small}
\begin{align*}
|T^cv_c(s) - T^c v'_c(s)|
    & = |\E_{A\subseteq B} T^ev_e(s\circ A)
            - \E_{A\subseteq B} T^ev'_e(s\circ A)| \\
    & = |\E_{A\subseteq B} (T^ev_e(s\circ A) - T^ev'_e(s\circ A)| \\
    & \leq |\E_{A\subseteq B} \gamma (v_e(s\circ A) - v'_e(s\circ A)|\\
    & = \gamma |\E_{A\subseteq B} v_e(s\circ A) 
                   - \E_{A\subseteq B} v'_e(s\circ A)|\\
    & = \gamma |v_c(s) - v'_c(s)|
.
\end{align*}
\end{small}%

\section{Guarantees on the ADS Assumption}
\label{sec:efficientimp}

Let $\mu(s)(A)$ be the top action according to the DL $\mu(s)$.
We have the following theorem, which is a direct application of Hoeffding's concentration inequality along with a union bound.
\begin{lemma}
\label{lemma:concentration}
Let $Q^\star$ be optimal for the MDP. Suppose that for each state $s$ we sample $A_s^{(1)},\ldots,A_s^{(t)}$. If
\[
m = \Omega\left( \left(\frac{\gamma}{1-\gamma}\right)^2 \frac{|A| + \log(|S|/\delta)}{\epsilon^2} \right)~,
\]
then with probability $1-\delta$, for each DL policy $\mu$ we have
\[
\left| \E_{A \subseteq B} Q^\mu(s, \mu(s)(A)) - \frac{1}{T} \sum_{t=1}^T Q^\mu(s, \mu(s)(A_s^{(t)})) \right| \le \frac{\epsilon (1 - \gamma)}{2\gamma}~.
\]
\end{lemma}

We can use the lemma above to approximate the $Q$ function by the one corresponding to an approximate-MDP, resulted by the sub-sampling of the action subsets.

\begin{lemma}
\label{lemma:qapprox}
Let $\widehat{Q}^\mu$ be the $Q$ function corresponding to policy $\mu$ for the approximate-MDP. Then, for any policy $\mu$, state $s$ and action $a$,
\[
\left|Q^\mu(s,a) - \widehat{Q}^\mu(s,a) \right| \le \frac{\epsilon}{2}~.
\]
\end{lemma}

\begin{proof}
We will show one direction of the proof, and the other will follow by a symmetric argument. Let us unfold the $Q$ function by successive applications of \cref{lemma:concentration},
\begin{align*}
Q^\mu(s,a) &= \E_{s'|s,a} \left[r(s,a) + \gamma \E_{A \subseteq B} Q^\mu(s',\mu(s')(A))\right] \\
&\le \E_{s'|s,a} \left[r(s,a) + \gamma \cdot \frac{1}{T} \sum_{t=1}^T \left( Q^\mu(s',\mu(s')(A_{s'}^{(t)})) + \frac{\epsilon(1-\gamma)}{2\gamma} \right) \right] \\
&= \E_{s'|s,a} \left[r(s,a) + \frac{\gamma}{T} \sum_{t=1}^T Q^\mu(s',\mu(s')(A_{s'}^{(t)})) \right] + \frac{\epsilon(1-\gamma)}{2}~.
\end{align*}
Repeating the argument above with respect to $Q^\mu(s',\mu(s')(A_{s'}^{(t)}))$ and continuing so recursively, we obtain
\[
Q^\mu(s,a) \le \widehat{Q}^\mu(s,a) + \sum_{t=0}^\infty \gamma^t \frac{\epsilon(1-\gamma)}{2} = \widehat{Q}^\mu(s,a) + \frac{\epsilon}{2}~.
\]
\end{proof}

We can now state our theorem.

\begin{thm}
Let $\hat{\mu}$ be the optimal policy for the approximate-MDP. 
For each state $s$ and action $a$, $Q^{\hat{\mu}}(s,a) \ge Q^\star(s,a) - \epsilon$.
\end{thm}

\begin{proof}
We have,
\begin{align*}
Q^\star(s,a) &\le \widehat{Q}^{\mu^\star}(s,a) + \frac{\epsilon}{2} & &\mbox{(\cref{lemma:qapprox})} \\
&\le \widehat{Q}^{\hat{\mu}}(s,a) + \frac{\epsilon}{2} & &\mbox{($\hat{\mu}$ optimal for $\widehat{Q}$)} \\
&\le Q^{\hat{\mu}}(s,a) + \epsilon~. & &\mbox{(\cref{lemma:qapprox})}
\end{align*}
\end{proof}

\section{Non-stochastic Action Availability}

Suppose that the action subsets are chosen by an adversarial process, yet are known in advance. In this case the optimal policy might neither be stationary nor a decision list. This is shown in the following example.

\begin{center}
\begin{tikzpicture}
\node[draw, circle] at (0,0) {$1$};
\node[draw, circle] at (2,0) {$2$};
\node[draw, circle] at (4,0) {$3$};
\draw[->, >=latex] (-0.3,0.1) arc (0:330:0.3cm) node[left=20, above] {$a$};
\draw[->, >=latex] (0.3,0) -- (1.7,0) node[above, pos=0.5] {$b$};
\draw[->, >=latex] (2.3,0) -- (3.7,0);
\draw[->, >=latex] (4,-0.3) arc (-20:-160:2.1);
\draw[->, >=latex] (0.1,0.3) arc (160:20:2) { node[above=15,left=100] {$c$} };
\end{tikzpicture}
\end{center}

Consider the MDP drawn above, which has three states. When the agent is at state 1, she has the choice of playing actions $a$, $b$ or $c$. Actions $b$ and $c$ are always available, and $b$ generates a slightly higher reward than $c$\footnote{Compared to the discounting factor.}. Action $a$ is available once in a while, and its reward is much higher than both $b$ and $c$.

When the agent visits state 1 and action $a$ is unavailable, she will usually prefer playing action $b$. However, should the agent know that action $a$ will become available in the next turn, she will play action $c$. This is even though in both cases the set of available actions is the same!

\section{The Complexity of Value Iteration}

Given its polytime per-iteration complexity, to ensure VI 
is polytime, we must show that it converges to a value function
that induces an optimal policy in a polynomially many iterations.
To do so, we exploit the compressed representation and adapt
the technique of \cite{tseng:ORLetters90}.

First we need some assumptions and definitions w.r.t.\ the base MDP $\calM$:
\begin{itemize}\denselist
\item Assume (w.l.o.g.) that all rewards/costs are integers.
\item Let $\delta$ be the smallest integer s.t.\ $\delta \gamma$ and
	$\delta p^k_{s,s'}$ (for all actions, states) are integer and
	$\delta > r^k_s $ for all states, actions.
This is the precision needed to represent the base:
each parameter requires $\log \delta$ bits.
\end{itemize}

Tseng shows that VI for a standard MDP is weakly polynomial (assuming
a constant discount factor $\gamma$), by proving that:
(a) if the $t$th value function produced by VI satisfies
$$
||V^t - V^\ast|| < 
1/(2 \delta^{2n+2} n^n)
,
$$
then the policy induced by $V^t$ is optimal;
and (b) VI achieves this bound in polynomially many
iterations.

Tseng's proof involves several steps:
\begin{itemize}\denselist
\item 
A simple argument based on the precision of the parameters shows that
for all $s$, $V^\ast(s) = w_s/(\delta^{2n}n^n)$ for some integer $w_s$:
Let $\mu$ be any policy.
Then the induced transition matrix $P^\mu$ and the immediate reward vector
$R^\mu$ are such that $Z=\delta^2(I - \gamma P^\mu)$ and 
$\delta^2 R^\mu$ are integral (and the entries of $Z$ are less than $\delta^2$).
Then for the optimal policy $\mu$,
by Cramer's rule we have that $V^\ast(s) = \det Z_s / \det Z$ (where
$Z_s$ is $Z$ with $R^\mu$ replacing column $s$).
By Hadamard's inequality, the determinant of $Z$ is bounded 
by $\delta^{2n} n^{n/2}$,
and since the determinant of $Z_s$ must also be integral,
the stated fact follows. 

\item
An
action elimination argument based on the precision of the solution 
shows that the action backup of any nonoptimal action w.r.t.\ $V^\ast$ must 
differ from that of the optimal action by at least $1/(\delta^{2n+2}n^n)$:
Suppose action $k$ is not optimal at $s$, hence
    $Q^*(s,k) = T^k V^\ast(s) < T^\mu V^\ast(s)$ (where $\mu$ is the
    optimal policy). Then
     \begin{align*}
      T^kV^\ast(s)
         & = \gamma \sum_{s'} p^k_{ss'} V^\ast(s) + r^k_s \\
         & = \gamma \sum_{s'} p^k_{ss'} w_s/(\delta^{2n}n^n) + r^k_s \\
         & = \frac{\delta^2 \gamma \sum_{s'} p^k_{ss'} w_s 
                 + \delta^{2n+2} n^{n/2}r^k_s}
                  {\delta^{2n+2}n^n}
.
     \end{align*}
Since the numerator is integral (and so is $w_s$),
$T^k V^\ast(s)$ and $T^\mu_c V^\ast(s)$ must differ by at least 
$1/(\delta^{2n}n^n)$.
Thus if $||V^t - v^\ast|| < 1/(2\delta^{2n}n^n)$ 
(where $V^t$ is $t$th iterate of VI),
some simple substitutions show that the policy induced by $V^t$ must be optimal.
\item
For any target precision $\veps$, Tseng uses standard arguments to show
  that after $t$ iterations, the error $||V^t - v^\ast|| < \veps$,
where $t = \lceil \log (||V^\ast||/(1-\gamma)\veps) / \log(1/\gamma)\rceil$.
We can plug in the usual upper bound $U$ on $V^\ast$,
which is $U = r_{\max} / (1-\gamma) = O(r_{\max})$ for a fixed
$\gamma$.
\item
Substituting the precision $1/(2\delta^{2n}n^n)$ for $\veps$
gives a polytime bound on required iterations:
$$
t^* \leq \log(2\delta^{2n}n^n U) / \log(1/\gamma)
.
$$
\end{itemize}

We derive a similar bound on the number of VI iterations
needed for convergence in an SAS-MDP,
using the same input parameters as in the base MDP, and applying the
the same precision $\delta$ to the action availability probabilities.
We apply Tseng's result by exploiting the fact that: (a) the optimal
policy for the embedded MDP $\calM_e$ can be represented as a DL; (b) the
transition function for any DL policy can be expressed using an $n\times n$
matrix (we simply take expectations, see above); and (c) the
corresponding linear system can be expressed over the \emph{compressed}
rather than the embedded state space 
to determine $V_c^\ast$ (rather than $V_e^\ast$). 

Tseng's argument requires some adaptation to apply to
the compressed VI algorithm.  
His definition of $\delta$ 
ensures that the terms in the linear system when multiplied by $\delta^2$
are integers, which in turn ensures the solution for $V^\ast$
has precision limited to $\delta^{2n}n^n$.
We extend his precision assumption to account for our action availability
probabilities as well, 
ensuring $\delta \rho^k_s$ is integral for all $s\in S, k\in B$. 

Since $\calM_c$ is an MDP, Tseng's result applies; but notice that
the transition matrix for any state's DL $\mu$, which serves
as an action in $\calM_c$, has entries of the form:
{\small
$$
p^\mu_{s,s'} =
  \sum_{i=1}^{m} \,[\prod_{j=1}^{i-1} (1-\rho^{\mu(s)(j)}_s)]\, \rho^{\mu(s)(i)}_s
 p^{\mu(s)(i)}_{s,s'}.
$$
}%
Since this is the product of $m+1$ probabilities, each with precision
$\delta$, we have that $\delta^{m+1} p^\mu_{s,s'}$ must also be integer.
Thus the required precision parameter for our MDP is at most
$\delta^{m+1}$. Plugging this into Tseng's bound, 
VI applied to $\calM_c$ must induce an optimal policy at the $t$th iteration if
{\small
$$
||V^t - v^\ast|| < 1/(2\delta^{(m+1)^{2n}} n^n)
                   = 1/(2\delta^{(m+1)2n} n^n)
.
$$
}%
This in turn gives us a bound on the number of iterations of VI
needed to reach an optimal policy:
\begin{thm} 
VI applied to $\calM_c$ converges to a value function whose greedy policy 
is optimal in $t*$ iterations, where
{\small
$$
t^* \leq \log(2\delta^{2n(m+1)} n^n M) / \log(1/\gamma)
.
$$
}
\end{thm}
Combined with Obs.~\ref{obs:VIperiteration}, we have:
\begin{cor} 
VI yields an optimal policy for the SAS-MDP corresponding to $\calM_c$
in polynomial time.
\end{cor}

We remark that in the ADS model, we only obtain an approximation to
the optimal policy. In fact, one cannot compute an exact
optimal policy without observing the entire support of the availability
distributions, which requires an exponential sample size.

Under the product distribution assumption, the theorem particularly implies that the convergence of policy iteration in strongly-polynomial time as long as each of the actions is available with constant probability. In this case, for any $\pi, a_i,a_j,s$ we have $\eta_\pi(s,a_i,a_j) \ge p_i p_j = \Omega(1)$, which in turn implies that we can take $q$ to be $\Omega(1)$ in the bound above.

\section{The Complexity of Policy Iteration}
This is an adaptation of the Hansen, Miltersen, Zwick paper.

\subsection{Policy Iteration}

Let $\mu(A|s)$ be the probability of having available actions $A$ in state $s$.
Denote $\sigma$ to be a permutation of the actions, and $\sigma(A)$ is the first available action in $A$ according to $\sigma$. Also denote $\eta_{\pi}(s,a_1,a_2) = \mu\{A : \pi^\star(s)(A) = a_1, \pi(s)(A) = a_2 | s\}$, the probability that $\pi^\star$ plays $a_1$ and simultaneously $\pi$ plays $a_2$.

\begin{algorithm}
\caption{Policy Iteration}
\begin{algorithmic}
\STATE Initialize: $\pi_1$
\FOR{$t=1,2,\ldots$}
\STATE Compute: $V^{\pi_t} = \E_{A|s} \left[ r(s, \pi_t(s)(A))+ \gamma \E_{s'|s,\pi_t(s)(A)} V^{\pi_t}(s') \right]$
\STATE Update: $\pi_{t+1}(s) = \argmax_{\sigma} \E_{A|s} \left[ r(s, \sigma(A))+ \gamma \E_{s'|s,\sigma(A)} V^{\pi_t}(s') \right]$
\ENDFOR
\end{algorithmic}
\end{algorithm}

\begin{thm}
\label{thm:vi}
$\|V^{\pi_{t+1}} - V^\star \| \le \gamma \|V^{\pi_t} - V^\star\|$.
\end{thm}

\begin{proof}
For any $s$,
\begin{align*}
V^\star(s) - V^{\pi_{t+1}}(s) &\le \gamma \max_\sigma \E_{A|s} \E_{s|s',\sigma(A)} \left( V^\star(s') - V^{\pi_t}(s') \right) \\
&\le \gamma \|V^{\pi_t} - V^\star\|~. 
\end{align*}
\end{proof}

\subsection{Polynomial bound}

In this section we will prove the following theorem.

\begin{thm}
\label{thm:strongpoly}
Assume there exists $q > 0$ such that for any policy $\pi$, state $s$ and actions $a_1,a_2$, $\eta_{\pi}(s,a_1,a_2) > 0$ implies $\eta_{\pi}(s,a_1,a_2)\ge q$. Then policy iteration converges after
\[
O \left( \frac{|S| |A|^2}{1-\gamma} \log \frac{|A|}{1-\gamma} \left( 1 + \log \frac{1}{q} \right) \right)
\]
iterations.
\end{thm}

Under the product distribution assumption, the theorem particularly implies that the convergence of policy iteration in strongly-polynomial time as long as each of the actions is available with constant probability. In this case, for any $\pi, a_i,a_j,s$ we have $\eta_\pi(s,a_i,a_j) \ge p_i p_j = \Omega(1)$, which in turn implies that we can take $q$ to be $\Omega(1)$ in the bound above.

\subsection{Proof Theorem~\ref{thm:strongpoly}}

\begin{lemma}
\label{lemma:difflb}
For any policy $\pi$, actions $a_1$, $a_2$, state $s$:
\[
V^\star(s) - V^\pi(s) \ge \eta_{\pi}(s,a_1,a_2) (Q^\star(s, a_1) - Q^{\star}(s,a_2))~.
\]
\end{lemma}

\begin{proof}
\begin{align}
V^\star(s) - V^\pi(s) &= V^\star(s) - \E_{A|s} \left[ r(s, {\pi}(s)(A))+ \gamma \E_{s'|s,{\pi}(s)(A)} V^{\pi}(s') \right]\\
&\ge V^\star(s) - \E_{A|s} \left[ r(s, {\pi}(s)(A))+ \gamma \E_{s'|s,{\pi}(s)(A)} V^\star(s') \right] \\
&= V^\star(s) - \E_{A|s} Q^\star(s, \pi(s)(A)) \\
&= \E_{A|s} \left[ Q^\star(s,\pi^\star(s)(A)) - Q^\star(s, \pi(s)(A)) \right] \\
&= \sum_A \mu\{A | s\} \left( Q^\star(s,\pi^\star(s)(A)) - Q^\star(s, \pi(s)(A)) \right) \\
&= \sum_A \mu\{A | s\} \sum_{a'_1,a'_2} \ind_{[a'_1 = \pi^\star(s)(A), a'_2=\pi(s)(A)]} \left( Q^\star(s,a'_1) - Q^\star(s, a'_2) \right) \\
&= \sum_{a'_1,a'_2} \sum_A \mu\{A | s\} \ind_{[a'_1 = \pi^\star(s)(A), a'_2=\pi(s)(A)]} \left( Q^\star(s,a'_1) - Q^\star(s, a'_2) \right) \\
&= \sum_{a'_1,a'_2} \underbrace{\mu\{A : a'_1 = \pi^\star(s)(A), a'_2=\pi(s)(A) | s\} }_{\eta_\pi(s,a'_1,a'_2)} \left( Q^\star(s,a'_1) - Q^\star(s, a'_2) \right) \\
&= \sum_{a'_1,a'_2} \eta_{\pi}(s,a'_1,a'_2) \left( Q^\star(s,a'_1) - Q^\star(s, a'_2) \right) \\
&\ge \eta_{\pi}(s,a_1,a_2) \left( Q^\star(s,a_1) - Q^\star(s, a_2) \right)~,
\end{align}
where the last inequality is due to the fact that $\pi^\star$ maximizes $Q^\star$ at any state, and for any given available action set $A$.
\end{proof}

\begin{lemma}
Let $\pi$ be any policy, and let $\bar{V} = V^{\star} - V^{\pi}$. Then $\bar{V}$ is the value function of following $\pi$ with respect to the rewards $\bar{r}(s,a) = V^\star(s) - Q^\star(s,a)$.
\end{lemma}

\begin{proof}
\begin{align}
\bar{V}(s) &= V^\star(s) - V^{\pi}(s) \\
&= V^\star(s) - \E_{A|s} \left[ r(s, \pi(s)(A))+ \gamma \E_{s'|s,\pi(s)(A)} V^{\pi}(s') \right] \\
&= V^\star(s) - \E_{A|s} \left[ Q^\star(s, \pi(s)(A)) - \gamma \E_{s'|s,\pi(s)(A)} V^\star(s')  + \gamma \E_{s'|s,\pi(s)(A)} V^{\pi}(s') \right] \\
&= \E_{A|s} \left[ V^\star(s) - Q^\star(s, \pi(s)(A)) + \gamma \E_{s'|s,\pi(s)(A)} \left( V^{\star}(s') - V^{\pi}(s') \right) \right] \\
&= \E_{A|s} \left[\bar{r}(s,\pi(s)(A)) + \gamma \E_{s'|s,\pi(s)(A)} \bar{V}(s') \right]~. 
\end{align}
\end{proof}

\begin{corollary}
\label{corr:diffub}
For any state $s$,
\[
V^\star(s) - V^\pi(s) \le \frac{|A|^2}{1-\gamma} \max_{s',a'_1,a'_2} \eta_{\pi}(s',a'_1,a'_2) \left( Q^\star(s',a'_1) - Q^\star(s',a'_2) \right)~.
\]
\end{corollary}

\begin{proof}
Following the previous lemma,
\begin{align}
V^\star(s) - V^\pi(s) &= \mathbb{E}_\pi \left[ \sum_{t=0}^\infty \gamma^t \bar{r}(s_t,a_t) \Biggr| s_0 = s\right] \\
&= \mathbb{E}_\pi \left[ \sum_{t=0}^\infty \gamma^t \E_{A|s_t} \bar{r}(s_t,\pi(s_t)(A)) \Biggr| s_0 = s\right] \\
&\le \sum_{t=0}^\infty \gamma^t \max_{s'} \left\lbrace \E_{A|s'} \bar{r}(s',\pi(s')(A)) \right\rbrace \\
&= \frac{1}{1-\gamma} \max_{s'} \sum_{a'_1,a'_2} \eta_{\pi}(s',a'_1,a'_2) \left( Q^\star(s',a'_1) - Q^\star(s',a'_2) \right) \\
&\le \frac{1}{1-\gamma} \sum_{a'_1,a'_2} \max_{s'} \eta_{\pi}(s',a'_1,a'_2) \left( Q^\star(s',a'_1) - Q^\star(s',a'_2) \right) \\
&\le \frac{|A|^2}{1-\gamma} \max_{s',a'_1,a'_2} \eta_{\pi}(s',a'_1,a'_2) \left( Q^\star(s',a'_1) - Q^\star(s',a'_2) \right)~.
\end{align}
\end{proof}

\begin{lemma}
\label{lemma:ratelb}
Let $\pi$,$\pi'$ be two policies. Let $(s',a'_1,a'_2) = \argmax_{s,a_1,a_2} \eta_{\pi}(s,a_1,a_2) (Q^\star(s,a_1) - Q^\star(s,a_2))$. 
Assume that $\eta_{\pi}(s',a'_1,a'_2) \le 2 \eta_{\pi'}(s',a'_1,a'_2)$, then
\[
\|V^\star - V^{\pi'}\| \ge \frac{1-\gamma}{2|A|^2} \| V^\star - V^{\pi} \|~.
\]
\end{lemma}

\begin{proof}
Combining \cref{lemma:difflb} and \cref{corr:diffub},
\begin{align*}
\|V^\star - V^{\pi}\| &\le \frac{|A|^2}{1-\gamma} \eta_{\pi}(s',a'_1,a'_2) \left( Q^\star(s',a'_1) - Q^\star(s', a'_2) \right) \\
&\le \frac{2|A|^2}{1-\gamma} \eta_{\pi'}(s',a'_1,a'_2) \left( Q^\star(s',a'_1) - Q^\star(s', a'_2) \right) \\
&\le \frac{2|A|^2}{1-\gamma} \|V^\star - V^{\pi'}\|~.
\end{align*}
\end{proof}

We are now ready to prove the theorem.

\begin{proof}[Proof of \cref{thm:strongpoly}]
Let $\pi_1,\ldots,\pi_T$ be the policies generated by policy iteration in $T$ iterations. Let $t < t'$ be some rounds, and denote $(s',a'_1,a'_2) = \argmax_{s,a_1,a_2} \eta_{\pi_t} (Q^\star(s,a_1) - Q^\star(s,a_2))$. If $\eta_{\pi_t}(s',a'_1,a'_2) \le 2 \eta_{\pi_{t'}}(s',a'_1,a'_2)$, then by \cref{lemma:ratelb}:
\[
\|V^\star - V^{\pi_{t'}}\| \ge \frac{1-\gamma}{2|A|^2} \| V^\star - V^{\pi_t} \|~.
\]
On the other hand, by \cref{thm:vi}:
\[
\|V^\star - V^{\pi_{t'}}\| \le \gamma^{t'-t} \| V^\star - V^{\pi_{t}} \|~,
\]
which means that $t' - t \le L$, where $L =  1/(1-\gamma) \log(2|A|^2/(1-\gamma))$.
Namely, after $\lfloor L \rfloor + 1$ iterations, we will have $\eta_{\pi_{t'}}(s',a_1',a_2') < (1/2) \eta_{\pi_{t}}(s',a_1',a_2')$.

The above argument holds for $a'_1,a'_2$ at most $\lfloor 1 + \log_2(1/q) \rfloor$ times, and thus for all states and actions at most $|S||A|^2 \lfloor 1 + \log_2(1/q) \rfloor$ times. Thus, the total number of iterations for policy iteration to converge is at most $4|S||A|^2 (1 + \log_2(1/q)) L$ as required.
\end{proof}
\end{document}